\documentclass{article}

\usepackage[preprint]{neurips_2026} % arXiv preprint

\usepackage[utf8]{inputenc}
\usepackage[T1]{fontenc}
\usepackage{hyperref}
\usepackage{url}
\usepackage{booktabs}
\usepackage{amsfonts}
\usepackage{amsmath}
\usepackage{amssymb}
\usepackage{mathtools}
\usepackage{amsthm}
\usepackage{nicefrac}
\usepackage{microtype}
\usepackage{xcolor}
\usepackage{graphicx}
\usepackage{enumitem}
\usepackage{multirow}
\usepackage{algorithm}
\usepackage[capitalize,noabbrev]{cleveref}

\theoremstyle{plain}
\newtheorem{theorem}{Theorem}
\newtheorem{proposition}[theorem]{Proposition}
\newtheorem{lemma}[theorem]{Lemma}
\newtheorem{corollary}[theorem]{Corollary}
\theoremstyle{definition}
\newtheorem{definition}[theorem]{Definition}
\newtheorem{assumption}[theorem]{Assumption}

\theoremstyle{remark}
\newtheorem{remark}[theorem]{Remark}

\newcommand{\E}{\mathbb{E}}
\newcommand{\Prob}{\mathbb{P}}
\newcommand{\R}{\mathbb{R}}
\newcommand{\N}{\mathbb{N}}
\newcommand{\calS}{\mathcal{S}}
\newcommand{\calA}{\mathcal{A}}

\newcommand{\calH}{\mathcal{H}}

\newcommand{\calN}{\mathcal{N}}
\newcommand{\indic}{\mathbf{1}}
\newcommand{\Doff}{\mathcal{D}_{N}}
\newcommand{\Poff}{\Prob^{\mathrm{off}}}
\newcommand{\Eoff}{\E^{\mathrm{off}}}

\newcommand{\IDSeta}{\mathrm{IDS}_{\eta}}
\newcommand{\IDSzero}{\mathrm{IDS}_{0}}
\newcommand{\TS}{\mathrm{TS}}
\newcommand{\refTS}{\mathrm{ref\mbox{-}TS}}
\newcommand{\BR}{\mathrm{BayesRegret}}
\newcommand{\Tr}{\operatorname{Tr}}
\newcommand{\argmin}{\operatorname*{arg\,min}}
\newcommand{\Gram}{\Lambda}
\newcommand{\Hent}{H_{2}}

\title{Information-Directed Offline-to-Online Reinforcement Learning}

\author{%
  Keru Chen \\
  School of Electrical, Computer and Energy Engineering, Arizona State University \\
  \texttt{kchen234@asu.edu} 
}

\begin{document}
\maketitle

\begin{abstract}
Decision-making from offline datasets typically warm-starts a policy or score model from fixed offline data and then refines it with limited online interaction. Offline data reduces uncertainty, but it does not remove the need for exploration; it changes what remains to be explored. We formalise this residual uncertainty by the conditional mutual information $I(\chi;\tau_{1:T}\mid\mathcal{D}_N)$ between a learning target $\chi$ and the online trajectories after conditioning on the offline dataset. This view leads naturally to information-directed sampling (IDS), a family parameterised by $\eta\ge 0$ that selects actions by trading off instantaneous regret against information gain. We prove a generic offline-to-online Bayesian regret bound for IDS through a ratio certificate: any information-ratio bound satisfied by a reference Thompson-sampling policy over the same randomised policy class is inherited by IDS. In a known-dynamics Bayesian linear-reward model, the conditional mutual information has a log-determinant form, and vanilla IDS ($\eta=0$) satisfies $\widetilde O\!\left(Hd\min\left\{\sqrt T,\,T\sqrt{C^\dagger_{\beta,\mathrm{IDS}_0}(N,T)/N}\right\}\right),$ where the coverage coefficient is tied to the visitation distribution induced by vanilla IDS itself. We also identify a warm-start regime with a dominated but informative probe in which vanilla IDS selects the probe while Thompson sampling never does, giving a constant-factor Bayesian regret separation. Controlled bandit experiments and D4RL offline-to-online RL experiments validate this mechanism: IDS is most beneficial when offline data is informative but leaves biased or low-probability residual uncertainty that targeted online actions can resolve, a regime shared by offline RL, offline black-box optimization, and Bayesian optimization.
\end{abstract}

%-----------------------------------------------------------------------
\section{Introduction}
\label{sec:intro}

Modern data-driven decision-making systems rarely learn entirely online. Whether the deployed pipeline is offline reinforcement learning, offline model-based black-box optimization, Bayesian optimization, or a contextual bandit, a model or policy is first warm-started from a fixed offline dataset and then adapted with a limited amount of online interaction~\citep{levine2020offline,kumar2020conservative,kostrikov2021iql,nair2020accelerating,nakamoto2023cal,yao2025lilodriver,lee2022offline,li2023proto,yao2026langmarl,ball2023efficient,zhao2025autonomous}. This offline-to-online pipeline is appealing because offline data makes online learning safer and cheaper~\citep{chen2024towards,yao2025comal}. However, offline data does not remove the need for exploration. It only changes what remains to be explored. After conditioning on the offline dataset, the learner no longer needs to explore broadly over the original state-action (or design) space, but must instead resolve the residual uncertainty left by the warm start. We focus on the offline-to-online RL setting, but the principle is community-agnostic.

Most existing offline-to-online methods treat the two phases through different principles. Offline algorithms are conservative: they avoid actions that are poorly covered by the behaviour policy because extrapolation error is the dominant failure mode~\citep{fujimoto2019off,kumar2020conservative,jin2021pessimism,xie2021bellman,rashidinejad2021bridging,uehara2022pessimistic}. Online algorithms are exploratory: they deliberately visit uncertain parts of the state-action space to improve future decisions~\citep{abbasi2011improved,osband2013psrl,jin2020linear}. The transition between these principles is usually handled by algorithmic devices such as replay balancing, policy constraints, calibrated pessimism, or exploration bonuses~\citep{nakamoto2023cal,li2023proto,lee2022offline,song2023hybrid}. These methods can work well, but they do not isolate the basic statistical question: after observing $\Doff$, which online actions are worth taking because they cheaply resolve the residual uncertainty left by the warm start?

We study offline-to-online learning from this residual-uncertainty perspective. Let $\chi$ denote a learning target, such as an unknown value parameter, model parameter, or latent task mode. The conditional mutual information
$
    I(\chi;\tau_{1:T}\mid\Doff)
$
measures how much information the online trajectories can still extract about $\chi$ after the offline dataset has already been observed. This quantity is not a penalty for poor offline coverage; it is the information still missing after the posterior has been conditioned on $\Doff$. The central claim of this paper is that offline-to-online exploration should be directed at this residual information rather than at uncertainty under the original prior.

This view points to the information-directed sampling family~\citep{russo2016information,russo2018ids}, which we write as $\IDSeta$ for the regularised rule and $\IDSzero$ for vanilla IDS. The family evaluates actions by a regret-information tradeoff
$$
    \Psi^\eta = \Delta^2/(g+\eta),
$$
where $\Delta$ is instantaneous expected regret and $g$ is information gain. Thompson sampling also admits information-ratio analyses~\citep{thompson1933likelihood,osband2013psrl,russo2014posterior}, but it explores implicitly by probability matching: it samples a model from the posterior and acts optimally for that model. After a strong offline warm start, this distinction matters. After offline warm start, the posterior often concentrates on a dominant mode while leaving low-probability but consequential modes unresolved. Thompson sampling tests that mode only when it samples it. The IDS family can instead choose a deliberately suboptimal diagnostic action if that action resolves the remaining uncertainty at sufficiently low regret.

Our Bayesian viewpoint is closest in spirit to~\citet{hu2024bayesian}, who use probability matching to navigate the pessimism-versus-optimism dilemma in off-to-on fine-tuning. We share the residual-posterior view but study an explicit regret--information tradeoff: a diagnostic action can be informative even when it is not optimal under any sampled model, so it is invisible to probability matching but selectable by IDS.

We make three contributions. (i)~A generic off-to-on Bayesian regret bound for $\IDSeta$ via a ratio certificate: a reference Thompson policy is used only as a witness, and the conditional mutual-information chain rule converts the sum of online information gains into the single residual-information term $I(\chi;\tau_{1:T}\mid\Doff)$. (ii)~A known-dynamics linear-reward instantiation in which residual information has a closed-form log-determinant, yielding
$$
    \BR_T^{\mathrm{off}}(\pi^{\IDSzero,w})
    \le
    \widetilde O\!\left(
    Hd\min\left\{
        \sqrt{T},\,
        T\sqrt{C^\dagger_{\beta,\IDSzero}(N,T)/N}
    \right\}
    \right),
$$
where the warm-start branch tracks a coverage coefficient pinned to the visitation distribution induced by $\IDSzero$ itself, rather than a static object pinned to an optimal policy. (iii)~A dominated-probe regime in which Thompson sampling pays at least $(1-p)(1-e^{-1})(\underline\Delta_++\underline\Delta_-)$ over a horizon of order $1/p$ while $\IDSzero$ pays at most $(1-p)c_0+pc_1$, giving a constant-factor separation that is structural rather than asymptotic.

Experiments test the same mechanism. A hidden-mode bandit exhibits the predicted phase transition as the offline posterior concentrates; a biased linear contextual bandit shows the $\IDSeta$ ratio improving over standard baselines under informative-but-biased warm start; and \textbf{ROID} (\emph{Residual-Optimized Information-Directed Sampling}), a deep $\IDSeta$ selector with a TD3+BC backbone~\citep{fujimoto2021minimalist} and bootstrapped Q-ensemble, achieves state-of-the-art on D4RL~\citep{fu2020d4rl} continuous-control off-to-on tasks.

%-----------------------------------------------------------------------
\section{Related Work}
\label{sec:related}

\paragraph{Offline RL and pessimism.} Conservative or behaviour-constrained value estimation underlies most offline-only methods, both algorithmically (BCQ~\citep{fujimoto2019off}, CQL~\citep{kumar2020conservative}, IQL~\citep{kostrikov2021iql}, AWAC~\citep{nair2020accelerating}, TD3+BC~\citep{fujimoto2021minimalist}) and analytically (pessimism in linear MDPs and Bellman-consistent function classes~\citep{jin2021pessimism,xie2021bellman,uehara2022pessimistic,rashidinejad2021bridging}). These methods penalise extrapolation but do not directly extend to the online phase.

\paragraph{Off-to-on RL.} Recent work on the two-phase pipeline includes balanced replay and pessimistic ensembles~\citep{lee2022offline}, calibrated offline pretraining~\citep{nakamoto2023cal}, iterative policy regularisation~\citep{li2023proto}, hybrid actor-critic with offline data~\citep{ball2023efficient,song2023hybrid}, adaptive online policies~\citep{zheng2022online,kim2023sample}, and analyses of how offline data should accelerate online learning~\citep{wagenmaker2023leveraging,xie2022policy}. Most closely related is the Bayesian design-principles framework of~\citet{hu2024bayesian}, which argues that probability matching is a principled way to avoid the pessimism-versus-optimism dilemma in offline-to-online fine-tuning and introduces BOORL. Our work shares the posterior-conditioning viewpoint but replaces probability matching with an explicit information-ratio criterion. This lets us analyse residual information through $I(\chi;\tau_{1:T}\mid\Doff)$, track a coverage coefficient tied to the visitation induced by IDS, and identify a dominated-probe regime in which IDS and Thompson sampling separate.

\paragraph{Information-directed sampling.} The information-ratio framework was introduced for posterior sampling by~\citet{russo2014posterior,russo2016information}; IDS was introduced for bandits by~\citet{russo2018ids}, extended to heteroscedastic noise by~\citet{kirschner2018ids}, and lifted to RL by~\citet{hao2022ids}; broader information-theoretic foundations appear in~\citet{lu2021bitbybit}; a closely related decision-estimation viewpoint appears in~\citet{foster2021statistical}. Compared to~\citet{hao2022ids}, who establish IDS regret bounds in the purely online linear-mixture and tabular regimes, our analysis (a)~operates in the off-to-on regime where the prior is itself the offline-induced posterior, (b)~tracks a coverage object pinned to IDS-induced visitation, and (c)~identifies a structural regime in which IDS strictly beats TS, which is not visible from any TS-symmetric analysis.

\paragraph{Cross-community connection.} Decision-making from offline datasets is also studied in offline model-based black-box optimization, Bayesian optimization, and contextual bandits, where a surrogate (or reward predictor) fit on logged data is later refined by online queries. The same structural question recurs: after conditioning on the offline dataset, where does the residual uncertainty live, and which online queries cheapest resolve it? The conditional-mutual-information view and the regret--information ratio are community-agnostic, and our bandit and D4RL experiments speak to both the bandits/BO and offline-RL audiences this workshop convenes.

%-----------------------------------------------------------------------
\section{Setup}
\label{sec:setup}

A finite-horizon episodic MDP $M=(\calS,\calA,H,P,r)$ has horizon $H$ and rewards in $[0,1]$. A policy is $\pi=(\pi_1,\dots,\pi_H)$ with $\pi_h:\calS\to\Delta(\calA)$. For parameter $w$ and policy $\pi$, $V_h^\pi(s;w)$ and $Q_h^\pi(s,a;w)$ are the standard value functions and $\pi^*_w$ denotes an optimiser. For the regret analysis, we work in a known-dynamics Bayesian linear-reward model. This setting retains the offline-to-online exploration problem while isolating uncertainty in the stage-local reward/observation channel: a known feature map $\phi:\calS\times\calA\to\R^d$ satisfies $\|\phi\|_2\le 1$, and for each stage $h\in[H]$ an unknown $w_h\in\R^d$ gives $Q_h^\pi(s,a)=\langle\phi(s,a),w_h^\pi\rangle$. The prior is $w_h\sim\calN(0,\lambda^{-1} I)$ with fixed precision $\lambda\ge 1$, independent across stages, and observations carry unit-variance Gaussian noise; we write $w=(w_1,\dots,w_H)$.

The offline dataset is $\Doff=\bigcup_h\mathcal{D}_{N,h}$ with $n_h\ge 1$ samples at stage $h$ and $N=\sum_h n_h$. The stage-$h$ Gram matrix is $\Gram_{h,N}:=\lambda I+\sum_{(s,a)\in\mathcal{D}_{N,h}}\phi(s,a)\phi(s,a)^\top$. The online phase consists of $T$ episodes indexed $t=1,\dots,T$; with $\phi_{t,h}:=\phi(s_{t,h},a_{t,h})$ the rank-$1$ recursion $\Gram_{h,N+t}=\Gram_{h,N+t-1}+\phi_{t,h}\phi_{t,h}^\top$ defines $\Gram_{h,N+T}$. Let $\tau_t$ be the episode-$t$ trajectory and $\calH_t:=\sigma(\Doff,\tau_{1:t-1})$, with posterior $\beta_t:=\Prob(w\in\cdot\mid\calH_t)$. For a candidate policy $\pi$ and learning target $\chi$, the per-episode regret and information gain are
\begin{equation*}
\Delta_t(\pi):=\E_{t,\pi}[V_1^{\pi^*_w}(s_{t,1};w)-V_1^\pi(s_{t,1};w)],\qquad
I_t^\chi(\pi):=I(\chi;\tau_t\mid\calH_t,\pi_t=\pi).
\end{equation*}
Throughout the information-ratio analysis, $\Pi$ denotes the randomised closure of the candidate policy class; in particular, the episode-level posterior-sampling mixture used as the reference comparator is feasible.
For $\eta\ge 0$, the regularised IDS policy is
\begin{equation*}
\pi_t^{\IDSeta,\chi}\in\argmin_{\pi\in\Pi}\Psi_t^\eta(\pi;\chi),\qquad
\Psi_t^\eta(\pi;\chi):=\frac{\Delta_t(\pi)^2}{I_t^\chi(\pi)+\eta}.
\end{equation*}
We write $\IDSeta$ for the rule with score $\Delta^2/(g+\eta)$ and call $\IDSzero$ \emph{vanilla} IDS; ``IDS'' without a subscript refers to this family. Vanilla $\IDSzero$ uses the convention $\Psi=+\infty$ when $I_t^\chi=0<\Delta_t$ and $\Psi=0$ when $\Delta_t=I_t^\chi=0$. Bayesian regret over $T$ online episodes conditional on $\Doff$ is $\BR_T^{\mathrm{off}}(\pi):=\Eoff[\sum_{t=1}^T\Delta_t(\pi_t)]$. The notation $\widetilde O(\cdot)$ hides factors that are polylogarithmic in $T,H,d,N$.

\paragraph{Scope of the regret-bound instantiation.} For the regret-bound results in Section~\ref{sec:bound}, we work in the stage-local linear-Gaussian observation model: the transition kernel $P$ is known and does not depend on $w$, and the mean rewards used to define regret remain bounded in $[0,1]$. The component of the stage-$h$ observation $o_{t,h}$ that carries information about $w_h$ is a scalar linear-Gaussian regression target
\begin{equation*}
y_{t,h}=\phi(s_{t,h},a_{t,h})^\top w_h+\epsilon_{t,h},\qquad \epsilon_{t,h}\sim\calN(0,1)\ \text{independent},
\end{equation*}
which we treat as the Bayesian observation channel for posterior updates on $w_h$. The next state $s_{t,h+1}$ is conditionally independent of $w$ given $(\calH_{t,h},a_{t,h},y_{t,h})$. The structural separation in Section~\ref{sec:separation} is a separate finite construction and does not rely on this Gaussian observation model.

\paragraph{Linear-reward value-difference inequality.} Under the linear-reward and known-dynamics assumptions just stated, a direct telescoping argument yields the following value-difference inequality, which is the only structural property of the linear-$Q$ parametrisation that the regret-ratio proof of Section~\ref{sec:bound} uses: for any two parameter vectors $u,v\in\R^{Hd}$ and any policy $\pi$,
\begin{equation}\label{eq:linQ-vd}
\bigl|V_1^\pi(s_1;u)-V_1^\pi(s_1;v)\bigr|
\le
\sum_{h=1}^H \E_{\pi,P}\!\bigl[\bigl|\phi(s_h,a_h)^\top(u_h-v_h)\bigr|\bigr],
\end{equation}
where the expectation is taken over the known transition kernel $P$ and the policy randomness. We refer to~\eqref{eq:linQ-vd} as the \emph{linear-reward value-difference inequality}.

%-----------------------------------------------------------------------
\section{An Off-to-On Bayesian Regret Bound for IDS}
\label{sec:bound}

The argument has three pieces. An information-ratio bound on a reference policy passes to IDS automatically because IDS optimises the same ratio. Cauchy--Schwarz turns a per-episode ratio bound into a sum of square-rooted information gains. The conditional-MI chain rule turns that sum into a single number, $I(\chi;\tau_{1:T}\mid\Doff)$, into which the offline data enters only through the conditioning. The first two pieces are standard for posterior-sampling analysis~\citep{russo2014posterior,russo2016information}; the change of executed policy and the routing through conditional MI are what make the argument work in the off-to-on regime. Full proofs are in Appendix~\ref{app:proofs-bound}.

For each $t$ draw $\widetilde w_t\sim\beta_t$ independently of $w$ and set $\pi_t^\refTS:=\pi^*(\widetilde w_t)$. In the known-dynamics linear-reward instantiation of Section~\ref{sec:setup}, the standard probability-matching identity applies to state-action functionals under the reference posterior sample (Lemma~\ref{lem:prob-match} in Appendix~\ref{app:proofs-bound}). We require an information-ratio condition on this reference policy.

\begin{assumption}[Reference information-ratio condition]\label{assum:refts-ratio}
There exist a measurable event $G$ in the augmented probability space that includes the auxiliary reference samples $(\widetilde w_t)_{t\le T}$, and a constant $C_\chi\ge 0$, such that, on $G$ and for every $t\le T$,
\begin{equation*}
\Delta_t(\pi_t^\refTS)^2\le C_\chi\,I_t^\chi(\pi_t^\refTS).
\end{equation*}
\end{assumption}

In the linear-$Q$ model we verify Assumption~\ref{assum:refts-ratio} with $C_w=O(Hd)$ via stage-local decomposition (Section~\ref{subsec:linQ-instance}; proof in Appendix~\ref{app:proofs-bound}).

\begin{lemma}[Ratio certificate]\label{lem:cert}
On $G$ and for every $\eta\ge 0$, $\Psi_t^\eta(\pi_t^{\IDSeta,\chi};\chi)\le\Psi_t^\eta(\pi_t^\refTS;\chi)\le C_\chi$.
\end{lemma}

The first inequality holds because IDS minimises $\Psi_t^\eta$; the second is Assumption~\ref{assum:refts-ratio} after dividing by $I_t^\chi+\eta$. The lemma is one line, but the role it plays is specific: once the executed policy is no longer probability-matching, the reference policy serves only as a witness, and the executed policy inherits its ratio bound over the same policy class. This is the only point in the argument where the choice of executed policy matters.

\begin{proposition}[Master inequality]\label{prop:generic-decomp}
Under Assumption~\ref{assum:refts-ratio}, for every $\eta\ge 0$,
\begin{equation*}
\BR_T^{\mathrm{off}}(\pi^{\IDSeta,\chi})\le\sqrt{T C_\chi\bigl(I(\chi;\tau_{1:T}\mid\Doff)+T\eta\bigr)}+HT\Poff(G^c).
\end{equation*}
\end{proposition}

The proof square-roots Lemma~\ref{lem:cert}, applies Cauchy--Schwarz to the sum over episodes, takes $\Eoff[\cdot]$ with Jensen on the right, and uses that $\pi_t^{\IDSeta,\chi}$ is $\calH_t$-measurable to invoke the conditional-MI chain rule~\citep[Thm.~2.5.2]{cover2006elements}: $\sum_t\Eoff[I_t^\chi]=I(\chi;\tau_{1:T}\mid\Doff)$. The full argument is in Appendix~\ref{app:proofs-bound}. The proposition is the conceptual core of the paper: the offline dataset enters only through the conditioning of one mutual information, and the two phases are unified at the level of the proof.

\subsection{Linear-\texorpdfstring{$Q$}{Q} Instantiation}\label{subsec:linQ-instance}

Two ingredients turn Proposition~\ref{prop:generic-decomp} into a closed-form bound: a closed form for $I(w;\tau_{1:T}\mid\Doff)$, and a bound on $C_w$. Both follow from the standard stage-local decomposition $I(w;\tau_t\mid\calH_t,\pi)=\sum_h I(w_h;o_{t,h}\mid\calH_{t,h},a_{t,h},\pi)$ (Lemma~\ref{lem:mi-decomp}).

\begin{lemma}[Closed-form per-stage information]\label{lem:closed-form-MI}
Under the known-dynamics linear-reward instantiation, $I(w_h;o_{t,h}\mid\calH_{t,h},a_{t,h})=\tfrac12\log(1+\phi_{t,h}^\top\Gram_{h,N+t-1}^{-1}\phi_{t,h})$.
\end{lemma}

\begin{lemma}[Reference ratio constant]\label{lem:Cw}
Assumption~\ref{assum:refts-ratio} holds (with the trivial event $G=\Omega$, i.e.\ $\Poff(G^c)=0$) with $C_w=H\Gamma^2=O(Hd)$, where $\Gamma=\sqrt{4d/\log 2}$.
\end{lemma}

The proof of Lemma~\ref{lem:Cw} combines the probability-matching identity, the linear-reward value-difference inequality~\eqref{eq:linQ-vd}, two applications of Cauchy--Schwarz, and the Bayesian second-moment identity $\E[\|\widetilde w_{t,h}-w_h\|_{\Gram_{h,N+t-1}}^2\mid\calH_t]=2d$, which holds because $w_h$ and $\widetilde w_{t,h}$ are independent posterior draws given $\calH_t$. Combining Lemmas~\ref{lem:closed-form-MI}--\ref{lem:Cw} with the rank-$1$ log-determinant identity~\citep[Lem.~11]{abbasi2011improved} yields
\begin{equation}\label{eq:cum-info-logdet}
I(w;\tau_{1:T}\mid\Doff)=\tfrac{1}{2}\sum_{h=1}^H\Eoff\bigl[\log\det(\Gram_{h,N+T})-\log\det(\Gram_{h,N})\bigr].
\end{equation}

\subsection{Two Log-Determinant Bounds and an Interpolation Theorem}

Two complementary bounds control the log-determinant. The elliptical-potential bound~\citep[Lem.~19.4]{lattimore2020bandit} gives $\log(\det\Gram_{h,N+T}/\det\Gram_{h,N})\le d\log(1+T/(d\lambda_{\min}(\Gram_{h,N})))$, which yields the standard $\sqrt{T}$ regret rate. The trace bound, obtained via $\log\det(I+M)\le\Tr(M)$~\citep{horn2012matrix}, gives $\log(\det\Gram_{h,N+T}/\det\Gram_{h,N})\le\sum_{t=1}^T\phi_{t,h}^\top\Gram_{h,N}^{-1}\phi_{t,h}$, which we control by a coverage coefficient pinned to IDS's own visitation.

\begin{definition}[$\IDSeta$-induced coverage]\label{def:ids-coverage}
Assume $n_h\ge 1$ for every $h\in[H]$. With
\begin{equation*}
\overline\Sigma_{h,T}^{\IDSeta}(\Doff):=\frac{1}{T}\sum_{t=1}^T\E\bigl[\phi_{t,h}\phi_{t,h}^\top\bigm|\Doff\bigr]
\end{equation*}
the time-averaged online visitation covariance under $\IDSeta$, define
\begin{equation*}
C_{\beta,h}^{\IDSeta}(N,T):=\lambda_{\max}\!\left((\Gram_{h,N}/n_h)^{-1/2}\;\overline\Sigma_{h,T}^{\IDSeta}\;(\Gram_{h,N}/n_h)^{-1/2}\right),\ \ 
C^\dagger_{\beta,\IDSeta}(N,T):=\frac{N}{H}\sum_{h=1}^H \frac{C_{\beta,h}^{\IDSeta}(N,T)}{n_h}.
\end{equation*}
We write $C^\dagger_{\beta,\IDSzero}(N,T)$ for the same quantity under vanilla $\IDSzero$ ($\eta=0$).
\end{definition}

This coverage object differs from the standard off-to-on coverage coefficient~\citep{rashidinejad2021bridging,uehara2022pessimistic} in two ways: it is pinned to the time-averaged visitation under the algorithm rather than under $\pi^*$, and it depends on $T$ and on the algorithm itself. Both differences are necessary in the off-to-on regime, because the algorithm's online visitation is what actually multiplies the inverse offline Gram matrix in the trace bound.

\begin{theorem}[Off-to-on Bayesian regret bound]\label{thm:baseline}
Under Assumption~\ref{assum:refts-ratio} (verified at $C_w=O(Hd)$), vanilla $\IDSzero$ satisfies
\begin{equation*}
\BR_T^{\mathrm{off}}(\pi^{\IDSzero,w})\le\widetilde O\!\Bigl(Hd\,\min\!\bigl\{\sqrt{T L_N(T)},\;T\sqrt{C^\dagger_{\beta,\IDSzero}(N,T)/N}\bigr\}\Bigr),
\end{equation*}
with $L_N(T):=\frac{1}{H}\sum_h\log(1+T/(d\lambda_{\min}(\Gram_{h,N})))$.
\end{theorem}

For regularised $\IDSeta$ with $\eta>0$, the analogous two-branch bound holds with $C^\dagger_{\beta,\IDSeta}$ in place of $C^\dagger_{\beta,\IDSzero}$ and an additive slack $T\sqrt{C_w\eta}$; the precise statement and proof are deferred to Corollary~\ref{cor:eta-pos} in Appendix~\ref{app:proofs-bound}.

Theorem~\ref{thm:baseline} follows by combining Proposition~\ref{prop:generic-decomp}, Lemma~\ref{lem:Cw}, equation~\eqref{eq:cum-info-logdet}, and the two log-determinant bounds, with a conditional-trace identity bounding the trace branch by $C^\dagger_{\beta,\IDSzero}/N$; see Appendix~\ref{app:proofs-bound}. As $T\to\infty$ the bound recovers the Bayesian linear-bandit rate $\widetilde O(Hd\sqrt{T})$, the natural extension of $\widetilde O(d\sqrt T)$~\citep{russo2014posterior} to horizon $H$; as $N\to\infty$ at fixed $T$ the warm-start branch shrinks at rate $\sqrt{C^\dagger_{\beta,\IDSzero}/N}$. What is new is that the executed policy is IDS, that the inequality is stated for an arbitrary target $\chi$, and that the warm-start branch tracks a coverage coefficient pinned to the algorithm's own visitation.

%-----------------------------------------------------------------------
\section{Structural Separation: When IDS Strictly Beats TS}
\label{sec:separation}

This section is separate from the randomized-closure guarantee: over a finite candidate set, IDS can select an available diagnostic probe that is never posterior-optimal, whereas TS cannot. Theorem~\ref{thm:baseline} controls worst-case Bayesian regret but does not, by itself, distinguish IDS from any other algorithm whose ratio is at most $C_\chi$, because TS satisfies the same bound. We now exhibit a structurally identifiable subclass of off-to-on problems in which IDS strictly outperforms TS by a constant factor. The mechanism is simple. Once warm-start has concentrated the posterior on a high-probability mode but a low-probability mode remains, the optimal action is almost always one thing, but the algorithm still does not know which thing. Posterior sampling has to commit to whichever mode its sample selects, so it can only resolve the residual mode by playing it, paying positive regret each time it does. An information-aware algorithm can instead pay the regret of a single suboptimal action whose only purpose is to disambiguate.

Let $\theta\in\{0,1\}$ index the residual mode after offline pretraining and let $p:=\Prob(\theta=1\mid\Doff)\in(0,1)$. Take $\chi:=\theta$ and restrict the candidate class to $\Pi^\dagger:=\{\pi_0,\pi_1,\pi_P\}$. The four assumptions, in compact form (formal statements in Appendix~\ref{app:proofs-separation}), are: cell-optimality \textbf{(B1)}, in which $\pi_0$ is uniquely optimal on $\theta=0$, $\pi_1$ is uniquely optimal on $\theta=1$, and $\pi_P$ is suboptimal on both; conditional gap constants $\underline\Delta_\pm,c_0,c_1>0$ \textbf{(B2)}; an information structure \textbf{(B3)} in which $\pi_0$ is uninformative about $\theta$ and $\pi_1,\pi_P$ are perfectly informative; and the $\IDSzero$ convention \textbf{(B4)} that $g=0\wedge\Delta>0\Rightarrow\Psi=+\infty$.

\begin{theorem}[Structural separation]\label{thm:separation}
Under \textnormal{\textbf{(B1)}--\textbf{(B4)}} and $\Pi=\Pi^\dagger$, TS never selects $\pi_P$ and the first time it plays $\pi_1$ is $T^*\sim\mathrm{Geom}(p)$. If $(1-p)c_0+pc_1<(1-p)\underline\Delta_-$, vanilla $\IDSzero$ selects $\pi_P$ at episode~$1$ and $\BR_T^{\IDSzero}(\Doff)\le(1-p)c_0+pc_1$ for every $T\ge 1$. For every $T\ge\lceil 1/p\rceil$, $\BR_T^\TS(\Doff)\ge(1-p)(1-e^{-1})(\underline\Delta_++\underline\Delta_-)$. If additionally $(1-p)c_0+pc_1<(1-p)(1-e^{-1})(\underline\Delta_++\underline\Delta_-)$, then $\BR_T^{\IDSzero}(\Doff)<\BR_T^\TS(\Doff)$ for every $T\ge\lceil 1/p\rceil$.
\end{theorem}

The proof (Appendix~\ref{app:proofs-separation}) computes the three values of $\Psi$ at episode~$1$ to establish IDS's choice, uses a geometric-hitting-time argument to lower-bound TS regret, and combines the two. The TS lower bound holds because TS by construction never plays $\pi_P$, regardless of how the posterior-sampling step is implemented: if $\widetilde\theta\in\{0,1\}$, then $\pi^*(\widetilde w)\in\{\pi_0,\pi_1\}$ by cell-optimality, so the dominated probe is structurally invisible to TS. The result is a structural separation, not a universal lower bound: it isolates a concrete off-to-on regime in which probability matching provably misses an informative dominated probe.

A warm-start corollary (Corollary~\ref{cor:warmstart}) makes the asymptotic gap explicit: if $c_0<\underline\Delta_-$ and $c_0<(1-e^{-1})(\underline\Delta_++\underline\Delta_-)$, the assumptions hold along any sequence of offline datasets with $p_N\downarrow 0$, and with $T_N:=\lceil 1/p_N\rceil$ the asymptotic regret ratio satisfies $\liminf_N \BR_{T_N}^\TS/\BR_{T_N}^{\IDSzero}\ge(1-e^{-1})(\underline\Delta_++\underline\Delta_-)/c_0>1$. An explicit explicit finite separation instance satisfying \textbf{(B1)}--\textbf{(B4)} is constructed in Appendix~\ref{app:concrete}.

%-----------------------------------------------------------------------
\section{ROID: A Practical Algorithm for Off-to-On Deep RL}
\label{sec:practical}

The analysis of Sections~\ref{sec:bound}--\ref{sec:separation} is in the known-dynamics Bayesian linear-reward model with a closed-form posterior. To run the same selection rule on continuous control, we need a finite candidate set and practical surrogates for $\Delta$ and $g$ computed from deep critics. We instantiate the regret--information selection rule as a concrete algorithm, \textbf{ROID} (\emph{Residual-Optimized Information-Directed Sampling}), and separate the theory-faithful linear implementation from the deep implementation, where ensemble uncertainty provides a practical posterior proxy.

\paragraph{Path A (linear contextual bandit, used for bandit experiments).}
We use a linear-Gaussian contextual bandit matching the analysis: $r(s,a)=\phi(s,a)^\top w^\star+\varepsilon$ with $\varepsilon\sim\mathcal{N}(0,\sigma^2)$ and a fixed feature map $\phi$. The offline data, collected by a mismatched behaviour policy, is used to warm-start a Bayesian linear regression posterior, which is updated online in closed form. All methods share the same posterior. IDS is implemented with a sampling-based approximation: we draw posterior samples to estimate $\Delta(a)=\E_w[\max_{a'} \phi(s,a')^\top w - \phi(s,a)^\top w]$, and compute $g(a)=\tfrac12\log(1+\phi(s,a)^\top\Lambda^{-1}\phi(s,a)/\sigma^2)$ in closed form. Actions are selected by minimising $\Delta(a)^2/(g(a)+\eta)$. This setting matches the assumptions of Theorem~\ref{thm:baseline}, so the bandit experiments directly test the theory.

\paragraph{Path B (deep ensemble, used for D4RL).}
For continuous control, we use a Bayesian ensemble backbone~\citep{hu2024bayesian} and modify only the online action selector. The offline stage trains a bootstrapped TD3+BC ensemble; online, this ensemble is used as a posterior surrogate over $Q$-functions. For each state $s$, we sample a finite actor-anchored candidate set $A_t$ and choose actions using an IDS-inspired score.

Specifically, for each candidate $a\in A_t$, we estimate
\[
\Delta(a)
\approx
\frac{1}{K}\sum_{k=1}^K
\Bigl(
V_k^\star(s)-\widehat Q_k(s,a)
\Bigr),
\qquad
g(a)
\approx
\frac12
\log\Bigl(
1+\alpha\,\mathrm{Var}_k[\widehat Q_k(s,a)]/\sigma^2
\Bigr),
\]
where $\widehat Q_k$ denotes a clipped ensemble critic and
$V_k^\star(s)$ is computed over an enlarged candidate set. The selected action is
\[
a_t
=
\argmin_{a\in A_t}
\frac{\Delta(a)^2}{g(a)+\eta}.
\]
Thus ROID preserves the regret--information form of IDS, but replaces the exact Bayesian posterior and closed-form information gain by ensemble-based surrogates. The D4RL experiment evaluates whether the ROID selection principle remains stable and effective when implemented with deep posterior surrogates, while the formal guarantees remain those of the linear-Gaussian setting. Full architectural details, clipping, calibration, replay mixing, and hyperparameters are given in Appendix~\ref{app:implementation}.

%-----------------------------------------------------------------------
\section{Experiments}
\label{sec:experiments}

The experiments test the predicted mechanism rather than establish a new offline-RL backbone. We check three claims: that an information-aware rule probes a cheap diagnostic action when the warm-start posterior leaves a low-probability mode unresolved (hidden-mode bandit); that the regret-information ratio improves over greedy and probability matching under biased warm starts (linear contextual bandit); and that the same rule integrates into a deep offline-to-online pipeline without instability (D4RL continuous control).

\subsection{Hidden-Mode Bandit}
\label{subsec:hidden-mode}

The first experiment isolates the separation mechanism of
Theorem~\ref{thm:separation}. A binary latent variable $\theta\in\{0,1\}$
selects one of two modes. The agent has three actions:
$a_0\equiv\pi_0$ (default), $a_1\equiv\pi_1$ (rare-mode), and
$a_P\equiv\pi_P$ (probe). The rewards are $r(a_0)=1$ in both modes,
$r(a_1)=0.2/2.0$, and $r(a_P)=0.85/1.85$, so the probe is suboptimal by
$0.15$ under both modes but reveals $\theta$ in one observation. The offline
dataset induces a posterior residual probability
$p_N:=\Prob(\theta=1\mid\Doff)$ that decreases with $N$. We run $T=500$
online steps and average over $10$ seeds. Since Theorem~\ref{thm:separation}
concerns vanilla $\IDSzero$, the main text reports $\IDSzero$; the full
$\eta$ sweep is in Appendix~\ref{app:bandit-extra}.

Panel (a) of Table~\ref{tab:bandit-main} shows the predicted phase transition as the
residual mode probability $p_N$ shrinks. For moderate residual uncertainty
($N\in\{100,200,300\}$), the probe has a favourable regret-information tradeoff,
and $\IDSzero$ pays essentially only the one-step probe cost. TS behaves
differently: because it explores by probability matching, it resolves the rare
mode only when that mode is sampled, which produces consistently larger regret.

The most diagnostic case is $N=1000$, where $p_N=0.0066$. Greedy commits to the
default action and never resolves the rare mode. UCB also fails because its
uncertainty bonus scales with the posterior standard deviation, which shrinks
as $\sqrt{p_N(1-p_N)}$. Vanilla $\IDSzero$ still probes because the boundary
convention $g=0<\Delta\Rightarrow\Psi=+\infty$ rules out actions that incur
positive regret while providing no information. This is the empirical signature
of Theorem~\ref{thm:separation}: the advantage of $\IDSzero$ comes from selecting
a dominated but informative action, not from generic optimism.

\subsection{Linear Contextual Bandit with Biased Warm-Start}
\label{subsec:linear-bandit}

The second experiment tests a less discrete version of the same mechanism: the
offline posterior is informative but biased. We use a linear contextual bandit
\[
r(s,a)=\phi(s,a)^\top w^*+\varepsilon,\qquad
\varepsilon\sim\calN(0,\sigma^2),
\]
with a nonlinear random feature map. Offline data is generated by a biased
behaviour policy $w_{\mathrm{behav}}=w^*+\beta\delta$ with $\beta=6$, so the
warm-start posterior is shifted toward the behaviour policy. A candidate set of
size $M=256$ is sampled once and held fixed across online steps, matching the
actor-anchored candidate set used later in Section~\ref{sec:practical}. State
and action dimensions are $16$ and $4$, the feature dimension is $128$,
$\sigma=0.05$, the online horizon is $T=200$, and results are averaged over
$20$ seeds.

\begin{table}[t]
\centering
\small
\setlength{\tabcolsep}{4pt}
\renewcommand{\arraystretch}{0.92}

\begin{minipage}[t]{0.49\linewidth}
\centering
{\footnotesize\textbf{(a) Hidden-mode bandit ($T=500$).}\par}
\vspace{0.3em}
\begin{tabular}{rccccc}
\toprule
$N$ & $p_N$ & greedy & UCB & TS & $\IDSzero$ \\
\midrule
100  & 0.377  & 0.15  & 0.50  & 1.06 & \textbf{0.15} \\
200  & 0.268  & 0.15  & 0.15  & 1.28 & \textbf{0.15} \\
300  & 0.182  & 0.15  & 0.15  & 1.29 & \textbf{0.15} \\
1000 & 0.0066 & 3.31  & 3.31  & 1.79 & \textbf{0.15} \\
\bottomrule
\end{tabular}
\end{minipage}\hfill
\begin{minipage}[t]{0.49\linewidth}
\centering
{\footnotesize\textbf{(b) Biased contextual bandit ($T=200$).}\par}
\vspace{0.3em}
\begin{tabular}{rcccc}
\toprule
$N$ & greedy & UCB & TS & $\mathrm{IDS}_{0.5}$ \\
\midrule
20  & 42.85 & 4.80  & 10.64 & \textbf{3.57} \\
50  & 18.90 & 1.94  & 4.73  & \textbf{1.67} \\
100 & 0.88  & \textbf{0.113} & 0.445 & 0.641 \\
\bottomrule
\end{tabular}
\end{minipage}

\vspace{0.4em}
\caption{Bandit experiments validating the residual-information mechanism.
Left: hidden-mode bandit showing the dominated-probe effect.
Right: biased contextual bandit showing robustness under warm-start bias.
Full $\eta$ sweeps are in Appendix~\ref{app:bandit-extra}.}
\label{tab:bandit-main}
\end{table}

Panel (b) of Table~\ref{tab:bandit-main} shows when the regularised rule helps. Under strongly biased warm-start ($N=20,50$), greedy exploits the wrong posterior preference and TS pays for matching the biased posterior; $\mathrm{IDS}_{0.5}$ improves over both and over UCB. When the warm start is already accurate ($N=100$), UCB is best and $\mathrm{IDS}_{0.5}$ over-regularises the ratio. The experiment thus supports the intended mechanism rather than a uniform dominance claim: regularised IDS is most useful when the offline posterior is informative but biased.

\subsection{D4RL Benchmark}\label{subsec:d4rl}

The D4RL experiment isolates the effect of the online action-selection rule in a deep offline-to-online pipeline. We keep the offline backbone fixed and replace only the online selector by ROID, the practical $\IDSeta$ rule introduced in Section~\ref{sec:practical}. We use the Path~B implementation: each run trains a $5$-member TD3+BC ensemble offline, repacks the twin critics as $K=10$ Q-functions, and then uses the regularised $\IDSeta$ selector. The selector changes only the online action choice; the offline representation and critic ensemble play the role of a posterior surrogate. We evaluate on six D4RL~\citep{fu2020d4rl} \texttt{-v2} locomotion tasks: medium and medium-replay variants of \textsc{hopper}, \textsc{walker2d}, and \textsc{halfcheetah}. Per-environment hyperparameters are given in Appendix~\ref{app:implementation}. Evaluation uses the Polyak target actor on $10$ deterministic rollouts every $5{,}000$ environment steps.

\begin{table}[t]
\centering\small
\setlength{\tabcolsep}{4pt}
\caption{D4RL normalised return (offline$\to$online). Bold marks the highest online value per row. ROID numbers are mean over 3 seeds. Baselines are ODT~\citep{zheng2022online}, PEX~\citep{zhang2023policy}, Cal-QL~\citep{nakamoto2023cal}, RLPD~\citep{ball2023efficient} and BOORL~\citep{hu2024bayesian}. RLPD is a from-scratch online method, so only its online value is reported. The ROID column is our result.}
\label{tab:d4rl-baselines}
\begin{tabular}{llcccccc}
\toprule
Task & Type & RLPD & ODT & PEX & Cal-QL & BOORL & \textbf{ROID (ours)} \\
\midrule
\multirow{2}{*}{Hopper}
 & medium        & 107.3 & 66.9$\to$97.5  & 63.8$\to$78.6  & 75.8$\to$100.6 & 61.9$\to$109.8 & 51.21$\to$\textbf{113.22} \\
 & medium-replay & 58.9  & 86.6$\to$88.8  & 89.8$\to$103.3 & 95.4$\to$106.1 & 75.5$\to$\textbf{111.1} & 56.47$\to$106.33 \\
\midrule
\multirow{2}{*}{Walker2d}
 & medium        & 108.6 & 72.1$\to$76.7  & 79.8$\to$94.8  & 80.8$\to$89.6  & 83.6$\to$107.7 & 82.63$\to$\textbf{110.60} \\
 & medium-replay & 115.0 & 68.9$\to$76.8  & 73.6$\to$89.3  & 83.8$\to$94.5  & 69.1$\to$114.4 & 80.50$\to$\textbf{131.58} \\
\midrule
\multirow{2}{*}{HalfCheetah}
 & medium        & 90.5  & 42.7$\to$42.1  & 47.3$\to$67.8  & 48.0$\to$72.3  & 47.9$\to$\textbf{98.7} & 47.48$\to$95.51 \\
 & medium-replay & 87.6  & 39.9$\to$40.4  & 44.1$\to$55.2  & 46.5$\to$59.5  & 44.5$\to$91.5  & 43.90$\to$\textbf{91.65} \\
\midrule
\multicolumn{2}{l}{\emph{Sum (online)}} & 567.9 & 422.3 & 489.0 & 522.6 & 633.2 & \textbf{648.9} \\
\bottomrule
\end{tabular}
\end{table}

Table~\ref{tab:d4rl-baselines} provides evidence that ROID improves online fine-tuning in this D4RL protocol. It obtains the best score on $4/6$ tasks and the highest summed online score, exceeding BOORL by $15.7$ points in total. The strongest improvement occurs on \textsc{walker2d-medium-replay}, where ROID reaches $131.58$, compared with $115.0$ for RLPD and $114.4$ for BOORL. The gains concentrate on medium-replay tasks, where the offline data is heterogeneous and residual uncertainty is more consequential; this is exactly the regime targeted by the regret-information criterion. On easier warm starts such as \textsc{halfcheetah-medium}, ROID remains competitive but does not dominate.

%-----------------------------------------------------------------------

\section{Conclusion}
\label{sec:conclusion}

Offline-to-online learning changes what remains uncertain after the offline dataset is observed. We formalised this residual uncertainty by $I(\chi;\tau_{1:T}\mid\Doff)$ and showed that IDS targets it through a regret--information tradeoff, yielding an off-to-on Bayesian regret bound with a two-branch known-dynamics linear-reward guarantee that interpolates between standard online learning and warm-start coverage along IDS-induced visitation, and a dominated-probe regime in which IDS strictly separates from Thompson sampling. Bandit and D4RL experiments support the same mechanism: IDS is most useful when offline data narrows the posterior but leaves biased or low-probability residual modes. The residual-information view is community-agnostic and offers a common language for decision-making from offline datasets across offline RL, offline black-box optimization, Bayesian optimization, and contextual bandits.

\newpage
\bibliographystyle{plainnat}
\bibliography{ids_off_to_on}

%==================== APPENDICES =========================
\appendix
\newpage

\section{Proofs from Section~\ref{sec:bound}}
\label{app:proofs-bound}

We give complete, self-contained proofs of the results in Section~\ref{sec:bound}. Throughout the appendix, all expectations are conditional on $\Doff$ unless otherwise noted; we write $\Eoff$ and $\Poff$ for $\E[\cdot\mid\Doff]$ and $\Prob(\cdot\mid\Doff)$.

\subsection{Probability matching}
\label{app:prob-match}

\begin{lemma}[Probability matching]\label{lem:prob-match}
Assume the transition kernel $P$ does not depend on $w$ (the known-dynamics linear-reward instantiation of Section~\ref{sec:setup}). For any deterministic function $F_h(s,a;\beta_t)$ of $(s,a)$ given $\beta_t$,
\begin{equation*}
\E_{t,\pi_t^\refTS}[F_h(s_{t,h},a_{t,h};\beta_t)]=\E_{w\sim\beta_t}\E_{\pi^*_w}[F_h(s_h^{\pi^*_w},a_h^{\pi^*_w};\beta_t)].
\end{equation*}
\end{lemma}

\begin{proof}
Conditional on $\beta_t$, the reference sample $\widetilde w_t$ is drawn from the same posterior as the true parameter $w$. Because the transition kernel does not depend on $w$, the trajectory law under the policy $\pi^*(\widetilde w_t)$ played in the true MDP depends on $\widetilde w_t$ only through the policy it selects and not through any separate dynamics parameter. Therefore
\begin{equation*}
\E_{t,\pi_t^\refTS}[F_h(s_{t,h},a_{t,h};\beta_t)]
= \int \E_{\pi^*_{\widetilde w}}\!\big[F_h(s_h^{\pi^*_{\widetilde w}},a_h^{\pi^*_{\widetilde w}};\beta_t)\big]\,d\beta_t(\widetilde w),
\end{equation*}
where the inner expectation is taken under the trajectory law of $\pi^*_{\widetilde w}$ in the (parameter-free) dynamics. Renaming the dummy variable $\widetilde w$ as $w$ yields the claim.
\end{proof}

\subsection{Stage-local decomposition}
\label{app:mi-decomp}

\begin{assumption}[Stage-local conditional independence]\label{assum:SLCI}
For every $t,h$, $o_{t,h}\perp w_{-h}\mid(w_h,\calH_{t,h},a_{t,h})$, where $w_{-h}:=(w_{h'})_{h'\neq h}$.
\end{assumption}

This assumption holds in the stage-local linear-Gaussian instantiation of Section~\ref{sec:setup}: the regression-target component $y_{t,h}=\phi_{t,h}^\top w_h+\epsilon_{t,h}$ depends on $w$ only through $w_h$, and the transition component $s_{t,h+1}$ does not carry additional information about $w$ since the dynamics are independent of $w$.

\begin{lemma}[Episode MI decomposition]\label{lem:mi-decomp}
Under Assumption~\ref{assum:SLCI}, for any $\calH_t$-measurable policy $\pi$, $I(w;\tau_t\mid\calH_t,\pi_t=\pi)=\sum_{h=1}^H I(w_h;o_{t,h}\mid\calH_{t,h},a_{t,h},\pi_t=\pi)$.
\end{lemma}

\begin{proof}
The trajectory factors as $\tau_t=(s_{t,1},a_{t,1},o_{t,1},\dots,a_{t,H},o_{t,H})$. Since $s_{t,1}\sim\rho_1$ is independent of $w$, $I(w;s_{t,1}\mid\calH_t,\pi)=0$. Applying the chain rule of conditional mutual information~\citep[Thm.~2.5.2]{cover2006elements} repeatedly,
\[
I(w;\tau_t\mid\calH_t,\pi)
= \sum_{h=1}^{H}\Big[I(w;a_{t,h}\mid\calH_{t,h},\pi) + I(w;o_{t,h}\mid\calH_{t,h},a_{t,h},\pi)\Big].
\]
Because $\pi$ is $\calH_t$-measurable and $a_{t,h}=\pi_h(s_{t,h})$ is generated using fresh independent randomness, $a_{t,h}\perp w\mid(\calH_{t,h},\pi)$, so $I(w;a_{t,h}\mid\calH_{t,h},\pi) = 0$ for every $h$. Applying Assumption~\ref{assum:SLCI} to each remaining term, under the conditioning $(\calH_{t,h},a_{t,h},\pi)$ the observation $o_{t,h}$ is independent of $w_{-h}$ given $w_h$, so $I(w;o_{t,h}\mid\calH_{t,h},a_{t,h},\pi) = I(w_h;o_{t,h}\mid\calH_{t,h},a_{t,h},\pi)$. Summing over $h$ gives the claim.
\end{proof}

\subsection{Closed-form per-stage MI}
\label{app:closed-form}

\begin{proof}[Proof of Lemma \ref{lem:closed-form-MI}]
Conditional on $\calH_{t,h}$, the prior on $w_h$ from Section~\ref{sec:setup} ($\calN(0,\lambda^{-1}I)$) updated by all offline samples and online observations through episode $t-1$ is again Gaussian: $w_h \mid \calH_{t,h} \sim \calN(\mu_{h,N+t-1},\Gram_{h,N+t-1}^{-1})$.

Given $a_{t,h}$, the component of $o_{t,h}$ that carries information about $w_h$ is the Gaussian regression target $y_{t,h} = \phi_{t,h}^\top w_h + \epsilon_{t,h}$ with $\epsilon_{t,h}\sim\calN(0,1)$; the remaining component (the next state $s_{t,h+1}$) is independent of $w_h$ given $(s_{t,h},a_{t,h})$ by the scope assumption of Section~\ref{sec:setup}. Hence $I(w_h;o_{t,h}\mid\calH_{t,h},a_{t,h}) = I(w_h;y_{t,h}\mid\calH_{t,h},a_{t,h})$. By the Gaussian-channel mutual-information identity~\citep[Thm.~9.1.1]{cover2006elements}, $I(X;X+Z)=\tfrac12\log(1+\sigma_X^2/\sigma_Z^2)$. Apply it with $X=\phi_{t,h}^\top w_h$ (variance $\phi_{t,h}^\top\Gram_{h,N+t-1}^{-1}\phi_{t,h}$) and $Z=\epsilon_{t,h}$:
\[
I(\phi_{t,h}^\top w_h;\, y_{t,h}\mid\calH_{t,h},a_{t,h})
= \tfrac{1}{2}\log\!\big(1 + \phi_{t,h}^\top\Gram_{h,N+t-1}^{-1}\phi_{t,h}\big).
\]
Finally, $y_{t,h}$ depends on $w_h$ only through the scalar projection $\phi_{t,h}^\top w_h$, so this projection is a sufficient statistic of $w_h$ for $y_{t,h}$ given $(\calH_{t,h},a_{t,h})$. The Markov chain $w_h \to \phi_{t,h}^\top w_h \to y_{t,h}$ together with the converse data-processing identity for sufficient statistics~\citep[Sec.~2.8]{cover2006elements} therefore gives
\[
I(w_h;y_{t,h}\mid\calH_{t,h},a_{t,h}) = I(\phi_{t,h}^\top w_h;\,y_{t,h}\mid\calH_{t,h},a_{t,h}).
\]
\end{proof}

\subsection{Reference ratio constant}
\label{app:Cw}

\begin{proof}[Proof of Lemma \ref{lem:Cw}]
We will show that
\begin{equation*}
\Delta_t(\pi_t^\refTS)
\le
\Gamma\sum_{h=1}^{H}\sqrt{\E_{t,\pi_t^\refTS}[J_{t,h}^w]},
\qquad
J_{t,h}^w := \tfrac12\log\!\bigl(1+\phi_{t,h}^\top\Gram_{h,N+t-1}^{-1}\phi_{t,h}\bigr),
\end{equation*}
with $\Gamma=\sqrt{4d/\log 2}$, and then deduce the ratio bound by Cauchy--Schwarz over stages. The proof uses no high-probability concentration event: it works at the level of posterior second moments.

\textit{Step 1 (regret to posterior linear differences).}
Let $\widetilde w_t=(\widetilde w_{t,h})_{h=1}^H$ be the posterior sample used by the reference TS policy, drawn independently of $w$ from $\beta_t$. Probability matching (Lemma~\ref{lem:prob-match}) gives
\[
\Delta_t(\pi_t^\refTS)
=
\E_t\!\bigl[V_1^{\pi^*_{\widetilde w_t}}(s_{t,1};\widetilde w_t) - V_1^{\pi^*_{\widetilde w_t}}(s_{t,1};w)\bigr],
\]
where the two value functions are evaluated under the same policy $\pi^*_{\widetilde w_t}$ at different model parameters. Applying the linear-reward value-difference inequality~\eqref{eq:linQ-vd} with $u=\widetilde w_t$, $v=w$, and $\pi=\pi^*_{\widetilde w_t}=\pi_t^\refTS$,
\begin{equation}\label{eq:Cw-after-VD}
\Delta_t(\pi_t^\refTS) \le \sum_{h=1}^{H}\E_{t,\pi_t^\refTS}\!\bigl[\bigl|\phi_{t,h}^\top(\widetilde w_{t,h}-w_h)\bigr|\bigr].
\end{equation}

\textit{Step 2 (Cauchy--Schwarz in the $\Gram$-norm).}
For each stage $h$, Cauchy--Schwarz in the $\Gram_{h,N+t-1}$-norm gives the pointwise bound
\[
\bigl|\phi_{t,h}^\top(\widetilde w_{t,h}-w_h)\bigr|
\le
\sqrt{\phi_{t,h}^\top\Gram_{h,N+t-1}^{-1}\phi_{t,h}}\;\cdot\;\|\widetilde w_{t,h}-w_h\|_{\Gram_{h,N+t-1}}.
\]
A second application of Cauchy--Schwarz to the joint expectation over the trajectory and the parameter pair yields
\begin{equation}\label{eq:Cw-cs-prod}
\E_{t,\pi_t^\refTS}\!\bigl[\bigl|\phi_{t,h}^\top(\widetilde w_{t,h}-w_h)\bigr|\bigr]
\le
\sqrt{\E_{t,\pi_t^\refTS}\!\bigl[\phi_{t,h}^\top\Gram_{h,N+t-1}^{-1}\phi_{t,h}\bigr]}\,\sqrt{\E_t\!\bigl[\|\widetilde w_{t,h}-w_h\|_{\Gram_{h,N+t-1}}^2\bigr]}.
\end{equation}

\textit{Step 3 (Bayesian second-moment identity).}
Conditional on $\calH_t$, $w_h$ and $\widetilde w_{t,h}$ are independent draws from the posterior $\calN(\mu_{h,N+t-1},\Gram_{h,N+t-1}^{-1})$, the former by definition of $\beta_t$ and the latter by independent posterior sampling. Hence $\widetilde w_{t,h}-w_h$ is conditionally Gaussian with mean zero and covariance $2\Gram_{h,N+t-1}^{-1}$, and (since $\Gram_{h,N+t-1}$ is $\calH_t$-measurable)
\begin{equation}\label{eq:Cw-2d}
\E_t\!\bigl[\|\widetilde w_{t,h}-w_h\|_{\Gram_{h,N+t-1}}^2\bigr]
=\Tr\!\bigl(\Gram_{h,N+t-1}\cdot 2\Gram_{h,N+t-1}^{-1}\bigr)
=2d.
\end{equation}
This step is exact, no concentration needed.

\textit{Step 4 (variance to information gain).}
Since $\|\phi\|\le 1$ and $\Gram_{h,N+t-1}\succeq\lambda I$ with $\lambda\ge 1$, $x:=\phi_{t,h}^\top\Gram_{h,N+t-1}^{-1}\phi_{t,h}\in[0,1]$. On $[0,1]$ the inequality $\log(1+x)\ge x\log 2$ holds (this is $2^x\le 1+x$, which follows from concavity of $1+x-2^x$ on $[0,1]$ and its vanishing at the endpoints), so
\begin{equation}\label{eq:Cw-x-to-J}
x \le \tfrac{2}{\log 2}\cdot \tfrac12\log(1+x)
=\tfrac{2}{\log 2}\,J_{t,h}^w.
\end{equation}
Combining \eqref{eq:Cw-cs-prod}, \eqref{eq:Cw-2d}, and~\eqref{eq:Cw-x-to-J},
\[
\E_{t,\pi_t^\refTS}\!\bigl[\bigl|\phi_{t,h}^\top(\widetilde w_{t,h}-w_h)\bigr|\bigr]
\le
\sqrt{\tfrac{2}{\log 2}\,\E_{t,\pi_t^\refTS}[J_{t,h}^w]}\cdot\sqrt{2d}
=\Gamma\sqrt{\E_{t,\pi_t^\refTS}[J_{t,h}^w]},\qquad\Gamma:=\sqrt{\tfrac{4d}{\log 2}}.
\]
Substituting into~\eqref{eq:Cw-after-VD},
\[
\Delta_t(\pi_t^\refTS)\le\Gamma\sum_{h=1}^{H}\sqrt{\E_{t,\pi_t^\refTS}[J_{t,h}^w]}.
\]

\textit{Step 5 (Cauchy--Schwarz over stages).}
By $(\sum_h a_h)^2\le H\sum_h a_h^2$,
\[
\Delta_t(\pi_t^\refTS)^2
\le H\Gamma^2\sum_{h=1}^{H}\E_{t,\pi_t^\refTS}[J_{t,h}^w]
=H\Gamma^2\,I_t^w(\pi_t^\refTS),
\]
where the equality uses Lemmas~\ref{lem:mi-decomp}--\ref{lem:closed-form-MI}. This is Assumption~\ref{assum:refts-ratio} with $G=\Omega$ and $C_w=H\Gamma^2=\frac{4Hd}{\log 2}=O(Hd)$.
\end{proof}

\subsection{Master inequality and interpolation theorem}
\label{app:master}

\begin{proof}[Proof of Proposition~\ref{prop:generic-decomp}]
By Lemma~\ref{lem:cert}, on $G$ and for every $t\le T$, $\Delta_t^2 \le C_\chi(I_t^\chi + \eta)$. Taking square roots and using $\sqrt{a+b}\le\sqrt a + \sqrt b$, then summing over $t$ and applying Cauchy--Schwarz,
\[
\sum_{t=1}^{T}\Delta_t\,\indic\{G\} \le \sqrt{T C_\chi}\,\sqrt{\sum_{t=1}^{T}(I_t^\chi + \eta)}.
\]
Apply $\Eoff[\cdot]$ with Jensen on the right. Since $\pi_t^{\IDSeta,\chi}$ is $\calH_t$-measurable, $\Eoff[I_t^\chi(\pi_t^{\IDSeta,\chi}) \mid \calH_t] = I(\chi;\tau_t\mid\calH_t)$, and the conditional-MI chain rule~\citep[Thm.~2.5.2]{cover2006elements} gives $\sum_t I(\chi;\tau_t\mid\calH_t) = I(\chi;\tau_{1:T}\mid\Doff)$. Since rewards are in $[0,1]$ over horizon $H$, $\Delta_t\le H$ on $G^c$, contributing at most $HT\Poff(G^c)$. Combining gives the proposition.
\end{proof}

\begin{proof}[Proof of equation~\eqref{eq:cum-info-logdet}]
By Lemma~\ref{lem:mi-decomp} and Lemma~\ref{lem:closed-form-MI}, $I(w;\tau_{1:T}\mid\Doff) = \sum_{h,t} \Eoff[\tfrac12\log(1 + \phi_{t,h}^\top\Gram_{h,N+t-1}^{-1}\phi_{t,h})]$. The rank-$1$ log-determinant identity~\citep[Lem.~11]{abbasi2011improved} gives $\log\det(A + vv^\top) = \log\det A + \log(1 + v^\top A^{-1}v)$, so the sum over $t$ telescopes to $\log\det\Gram_{h,N+T} - \log\det\Gram_{h,N}$, giving the claim.
\end{proof}

\begin{proof}[Proof of Theorem~\ref{thm:baseline}]
Specialising Proposition~\ref{prop:generic-decomp} to vanilla $\IDSzero$ ($\eta=0$) and combining with Lemma~\ref{lem:Cw} (which holds with $G=\Omega$, so the failure term $HT\Poff(G^c)$ vanishes) gives
\[
\BR_T^{\mathrm{off}}(\pi^{\IDSzero,w}) \le \sqrt{T\,C_w\,I(w;\tau_{1:T}\mid\Doff)},\qquad C_w=\tfrac{4Hd}{\log 2}=O(Hd).
\]

\textit{Branch 1 (elliptical potential).} The elliptical-potential bound gives $\tfrac12\sum_h \Eoff[\log(\det\Gram_{h,N+T}/\det\Gram_{h,N})] \le \tfrac{Hd}{2}L_N(T)$. Substituting and using $C_w=O(Hd)$ yields $\sqrt{T\cdot Hd\cdot Hd\,L_N(T)}=\widetilde O\bigl(Hd\sqrt{T L_N(T)}\bigr)$.

\textit{Branch 2 ($\IDSzero$-induced coverage).} The trace bound combined with the conditional-trace identity below gives $\tfrac12\sum_h \Eoff[\log(\det\Gram_{h,N+T}/\det\Gram_{h,N})] \le \tfrac{T d H}{2}\cdot C^\dagger_{\beta,\IDSzero}/N$, yielding $\sqrt{T\cdot Hd\cdot THd\,C^\dagger_{\beta,\IDSzero}/N}=\widetilde O\bigl(Hd\,T\sqrt{C^\dagger_{\beta,\IDSzero}(N,T)/N}\bigr)$.

\textit{Conditional-trace identity.} By linearity, $\Eoff[\sum_t \phi_{t,h}^\top\Gram_{h,N}^{-1}\phi_{t,h}] = T\Tr(\Gram_{h,N}^{-1}\overline\Sigma_{h,T}^{\IDSzero})$. Set $A := \Gram_{h,N}/n_h$ and $B := \overline\Sigma_{h,T}^{\IDSzero}$. Then $\Tr(\Gram_{h,N}^{-1} B) = n_h^{-1}\Tr(A^{-1/2}BA^{-1/2}) \le \tfrac{d}{n_h}\lambda_{\max}(A^{-1/2}BA^{-1/2}) = \tfrac{d}{n_h}C_{\beta,h}^{\IDSzero}$. Taking the minimum of the two branches gives the bound.
\end{proof}

\subsection{Regret bound for regularised \texorpdfstring{$\IDSeta$}{IDS-eta}}
\label{app:eta-pos}

\begin{corollary}[Regularised $\IDSeta$]\label{cor:eta-pos}
Under the conditions of Theorem~\ref{thm:baseline}, for every $\eta>0$, regularised $\IDSeta$ satisfies
\begin{equation*}
\BR_T^{\mathrm{off}}(\pi^{\IDSeta,w})
\le
\widetilde O\!\Bigl(
Hd\,\min\!\bigl\{
\sqrt{T L_N(T)},\;
T\sqrt{C^\dagger_{\beta,\IDSeta}(N,T)/N}
\bigr\}
\Bigr)
+
T\sqrt{C_w\eta},
\end{equation*}
where $C_w=4Hd/\log 2$ from Lemma~\ref{lem:Cw}. The price of regularisation is the additive slack $T\sqrt{C_w\eta}$, which vanishes as $\eta\downarrow 0$ and recovers Theorem~\ref{thm:baseline}.
\end{corollary}

\begin{proof}
Apply Proposition~\ref{prop:generic-decomp} at $\eta>0$. Since $\sqrt{a+b}\le\sqrt a+\sqrt b$ for $a,b\ge 0$, the master inequality splits as
\begin{equation*}
\BR_T^{\mathrm{off}}(\pi^{\IDSeta,w})
\le \sqrt{T C_w\,I(w;\tau_{1:T}\mid\Doff)} + T\sqrt{C_w\eta}.
\end{equation*}
The first term is treated exactly as in the proof of Theorem~\ref{thm:baseline}, with the only change that the trajectory law (and hence the visitation-induced coverage coefficient) is now under $\IDSeta$ rather than $\IDSzero$. The conditional-trace identity therefore yields the trace branch with $C^\dagger_{\beta,\IDSeta}(N,T)$ in place of $C^\dagger_{\beta,\IDSzero}(N,T)$; the elliptical-potential branch is unchanged. Combining gives the stated bound.
\end{proof}

\section{Proofs from Section~\ref{sec:separation}}
\label{app:proofs-separation}

\subsection{Formal assumptions}

\begin{assumption}[Cell-optimality, B1]\label{assum:B1}
\textbf{(B1.1)} For every $w$ with $\theta=0$, $\pi_0$ is the unique globally optimal policy. \textbf{(B1.2)} For every $w$ with $\theta=1$, $\pi_1$ is the unique globally optimal policy. \textbf{(B1.3)} $\pi_P$ is strictly suboptimal for every $w$ with $\theta\in\{0,1\}$.
\end{assumption}

\begin{assumption}[Conditional gaps, B2]\label{assum:B2}
A history $h$ is unresolved if $0<p_h:=\Prob(\theta=1\mid h,\Doff)<1$. Constants $\underline\Delta_+,\underline\Delta_-,c_0,c_1>0$ satisfy, for every unresolved $h$, $\E[V_1^*(w)-V_1^{\pi_0}(w)\mid h,\theta=1]\ge\underline\Delta_+$, $\E[V_1^*(w)-V_1^{\pi_1}(w)\mid h,\theta=0]\ge\underline\Delta_-$, $\E[V_1^*(w)-V_1^{\pi_P}(w)\mid h,\theta=0]\le c_0$, and $\E[V_1^*(w)-V_1^{\pi_P}(w)\mid h,\theta=1]\le c_1$.
\end{assumption}

\begin{assumption}[Information structure, B3]\label{assum:B3}
\textbf{(B3.1)} Under $\pi_0$ the trajectory law is identical on $\{\theta=0\}$ and $\{\theta=1\}$, so $I(\theta;\tau\mid\pi_0,h)=0$ and $\Prob(\theta=1\mid h,\tau,\pi_0)=p_h$ a.s. \textbf{(B3.2)} Each of $\pi_1,\pi_P$ is perfectly informative: $I(\theta;\tau\mid\pi_1,h)=I(\theta;\tau\mid\pi_P,h)=\Hent(p_h)$. Here $\Hent(p):=-p\log p-(1-p)\log(1-p)$ is the binary entropy in nats.
\end{assumption}

\begin{assumption}[Vanilla IDS convention, B4]\label{assum:B4}
$g=0<\Delta\Rightarrow\Psi=+\infty$; $g=\Delta=0\Rightarrow\Psi=0$.
\end{assumption}

\subsection{Proof of Theorem~\ref{thm:separation}}

\begin{proof}[Full proof]
\textit{Part (i): TS never plays $\pi_P$ and the discovery time is geometric.}
TS plays $\pi^*(\widetilde w)$ with $\widetilde\theta\in\{0,1\}$. By (B1.1)--(B1.2) the unique optimiser is $\pi^*(\widetilde w)=\pi_0$ if $\widetilde\theta=0$ and $\pi_1$ if $\widetilde\theta=1$, and (B1.3) rules out $\pi_P$. For the geometric law: on any TS trajectory in which only $\pi_0$ has been played up to (but not including) episode $t$, (B3.1) keeps the conditional distribution of $\theta$ given the history at $p$. Hence $\Prob(\text{TS plays } \pi_1 \text{ at episode } t \mid \text{not yet}) = p$ i.i.d. across $t$, so $T^*\sim\mathrm{Geom}(p)$.

\textit{Part (ii): IDS picks $\pi_P$ at episode $1$ and pays bounded regret.}
At episode $1$, $p_1=p\in(0,1)$. Compute each $\Psi$:
\begin{itemize}[leftmargin=*]
\item $\pi_0$: $\Delta(\pi_0) \ge p\underline\Delta_+ > 0$, $g(\pi_0)=0$, so $\Psi(\pi_0) = +\infty$ by (B4).
\item $\pi_1$: $\Delta(\pi_1) \ge (1-p)\underline\Delta_-$, $g(\pi_1) = \Hent(p)$, so $\Psi(\pi_1) \ge ((1-p)\underline\Delta_-)^2/\Hent(p)$.
\item $\pi_P$: $\Delta(\pi_P) \le (1-p)c_0 + p c_1$, $g(\pi_P) = \Hent(p)$, so $\Psi(\pi_P) \le ((1-p)c_0+p c_1)^2/\Hent(p)$.
\end{itemize}
The strict-probe condition $(1-p)c_0+pc_1<(1-p)\underline\Delta_-$ gives $\Psi(\pi_P) < \Psi(\pi_1)$. Combined with $\Psi(\pi_0)=+\infty$, IDS strictly prefers $\pi_P$.

By (B3.2), one episode of $\pi_P$ is perfectly informative, so the posterior on $\theta$ collapses to a Dirac on the true value. At every subsequent episode the history is resolved; the corresponding $\pi_{\theta_{\mathrm{true}}}\in\Pi^\dagger$ has $\Delta=0$ and $g=0$, so $\Psi=0$ by (B4). Therefore $\BR_T^{\IDSzero}(\Doff) \le \Delta(\pi_P) \le (1-p)c_0 + p c_1$.

\textit{Part (iii): TS lower bound.}
Decompose by $T^*$. For $1\le t<T^*$, TS plays $\pi_0$ with per-episode regret $\ge p\underline\Delta_+$. At $t=T^*$ (if $T^*\le T$), TS plays $\pi_1$, contributing $\ge (1-p)\underline\Delta_-$. So
\[
\BR_T^\TS(\Doff) \ge p\underline\Delta_+\E[\min(T^*-1,T)] + (1-p)\underline\Delta_-\Prob(T^*\le T).
\]
For $T^*\sim\mathrm{Geom}(p)$, $\Prob(T^*\le T) = 1-(1-p)^T$ and $\E[\min(T^*-1,T)] = (1-p)(1-(1-p)^T)/p$. Substituting, $\BR_T^\TS \ge (1-p)(1-(1-p)^T)(\underline\Delta_+ + \underline\Delta_-)$. For $T\ge\lceil 1/p\rceil$, $(1-p)^T \le e^{-1}$, so $\BR_T^\TS \ge (1-p)(1-e^{-1})(\underline\Delta_+ + \underline\Delta_-)$.

\textit{Part (iv).} Combine the upper bound of (ii) with the lower bound of (iii) under the separation condition.
\end{proof}

\subsection{Warm-start corollary and policy-class extension}

\begin{corollary}[Warm-start threshold]\label{cor:warmstart}
Let $\Doff$ vary with $N\in\N$ inducing $p_N\downarrow 0$, with the four assumptions uniform. If $c_0<\underline\Delta_-$ and $c_0<(1-e^{-1})(\underline\Delta_++\underline\Delta_-)$, then there is a threshold $N_\star$ such that for every $N\ge N_\star$, with $T_N:=\lceil 1/p_N\rceil$, $\liminf_{N\to\infty}\BR_{T_N}^\TS/\BR_{T_N}^{\IDSzero}\ge(1-e^{-1})(\underline\Delta_++\underline\Delta_-)/c_0>1$.
\end{corollary}

\begin{proof}
The strict-probe condition is continuous in $p$ and reduces as $p\to 0^+$ to $c_0<\underline\Delta_-$, which holds by assumption; similarly the separation condition reduces to $c_0<(1-e^{-1})(\underline\Delta_++\underline\Delta_-)$. Apply Theorem~\ref{thm:separation}; the lim-inf follows from $((1-p_N)(1-e^{-1})(\underline\Delta_++\underline\Delta_-))/((1-p_N)c_0 + p_N c_1) \to (1-e^{-1})(\underline\Delta_++\underline\Delta_-)/c_0$.
\end{proof}

\begin{corollary}[Domination under information-structure closure]\label{cor:unrestricted}
Let $\Pi\supseteq\Pi^\dagger=\{\pi_0,\pi_1,\pi_P\}$ and assume \textbf{(IS-Closure)}: for every $\pi\in\Pi$ and every unresolved $h$, $g_1(\pi;\theta\mid h)\in\{0,\Hent(p_h)\}$. Under \textbf{(B1)}--\textbf{(B4)} and the separation condition, $\BR_T^{\TS,\Pi}=\BR_T^{\TS,\Pi^\dagger}$, $\BR_T^{\IDSzero,\Pi}\le(1-p)c_0+pc_1$, and $\BR_T^{\IDSzero,\Pi}<\BR_T^{\TS,\Pi}$ for $T\ge\lceil 1/p\rceil$.
\end{corollary}

\begin{proof}
TS plays $\pi^*(\widetilde w)\in\Pi^\dagger$ by (B1), so $\BR_T^{\TS,\Pi}=\BR_T^{\TS,\Pi^\dagger}$. For IDS, let $\pi^\star\in\Pi$ be its episode-1 choice. Since $\pi_P\in\Pi$, $\Psi(\pi^\star)\le\Psi(\pi_P)<\infty$, and (IS-Closure) plus (B4) force $g_1(\pi^\star;\theta)=\Hent(p)$. Then $\Delta(\pi^\star)\le(1-p)c_0+pc_1$, the observation perfectly resolves $\theta$, and $\Psi=0$ thereafter. The TS lower bound from Theorem~\ref{thm:separation}~(iii) closes the argument.
\end{proof}

\begin{remark}\label{rem:isclosure-scope}
The closure condition is needed only when the online candidate class is enlarged beyond $\Pi^\dagger$. Without it, a partially informative policy $\pi_Q$ with $0<g_1(\pi_Q;\theta\mid h)<\Hent(p_h)$ and slightly smaller immediate regret than $\pi_P$ can satisfy $\Psi(\pi_Q)<\Psi(\pi_P)$, be selected by IDS, and leave residual uncertainty. The concrete construction in Appendix~\ref{app:concrete} satisfies the separation theorem on the restricted candidate class $\Pi^\dagger=\{\pi_0,\pi_1,\pi_P\}$. If the behaviour action $O$ from that construction is also included in the online candidate class, then \textbf{(IS-Closure)} need not hold, since $O$ can be partially informative about $\theta$.
\end{remark}

\section{A Concrete Linear-\texorpdfstring{$Q$}{Q} Instance}
\label{app:concrete}

We exhibit an explicit $H=2$ Bayesian linear-$Q$ instance that satisfies \textbf{(B1)}--\textbf{(B4)} for every large enough $N$. The state-action geometry is $\calS_1=\{s_1\}$, $\calS_2=\{\bot,g_0,g_1,y_+,y_-\}$, $\calA(s_1)=\{S,R,P,O\}$, with feature map sending each action to a distinct one-hot vector $\phi(s_1,\cdot)\in\{e_1,e_2,e_3,e_4\}\subset\R^4$ (so $d=4$). The stage-$1$ parameter takes the form $\theta_1^{(\theta)}=(1/2,\theta,1/2-c,0)$ with $c\in(0,1/2)$ and $\theta\in\{0,1\}$. Transitions are deterministic: $S,R\to\bot$, $P\to g_\theta$ (which itself reveals $\theta$), and under $O$ the next state is $y_-$ with probability $1$ if $\theta=0$ and $y_+$ with probability $q\in(0,1)$ (else $y_-$) if $\theta=1$. Stage-$2$ rewards distinguish the two cells.

This finite construction is used only for the structural separation result; the information that $R$ and $P$ each carry about $\theta$ is delivered through the reward (for $R$) and the deterministic next state (for $P$), and is therefore separate from the linear-Gaussian observation model used in the log-determinant regret bound of Section~\ref{sec:bound}.

Conditional on having observed $N$ no-$y_+$ trajectories under the behaviour policy that plays only $O$, the posterior over $\theta$ satisfies
\begin{equation*}
p_N=\frac{(1-q)^N}{1+(1-q)^N}\to 0.
\end{equation*}
At episode~$1$ of online IDS,
\begin{equation*}
\Delta(S)=p_N/2,\qquad \Delta(R)=(1-p_N)/2,\qquad \Delta(P)=p_N/2+c,\qquad \Delta(O)=1/2+p_N/2,
\end{equation*}
while $g(S)=0$ and $g(R)=g(P)=\Hent(p_N)$, since the reward of $R$ equals $\theta$ (revealing $\theta$) and the next state of $P$ is $g_\theta$ (also revealing $\theta$). IDS selects $P$ if and only if $p_N/2+c<(1-p_N)/2$, equivalently $p_N<1/2-c$, so $N\ge N_\star(q,c):=\min\{N:p_N<1/2-c\}$ is sufficient.

Identifying
\begin{equation*}
\pi_0\equiv S,\qquad \pi_1\equiv R,\qquad \pi_P\equiv P,
\end{equation*}
the constants of Assumption~\ref{assum:B2} are read off as
\begin{equation*}
\underline\Delta_+=\underline\Delta_-=\tfrac12,\qquad c_0=c,\qquad c_1=\tfrac12+c,
\end{equation*}
since $\E[V_1^*-V_1^{\pi_0}\mid\theta=1]=1-1/2=1/2$, $\E[V_1^*-V_1^{\pi_1}\mid\theta=0]=1/2-0=1/2$, $\E[V_1^*-V_1^{\pi_P}\mid\theta=0]=1/2-(1/2-c)=c$, and $\E[V_1^*-V_1^{\pi_P}\mid\theta=1]=1-(1/2-c)=1/2+c$. Consistently, $\Delta(\pi_0)=p_N\underline\Delta_+=p_N/2$, $\Delta(\pi_1)=(1-p_N)\underline\Delta_-=(1-p_N)/2$, and $\Delta(\pi_P)=(1-p_N)c_0+p_Nc_1=c+p_N/2$. The warm-start hypotheses of Corollary~\ref{cor:warmstart} reduce to $c<\underline\Delta_-=1/2$ and $c<(1-e^{-1})(\underline\Delta_++\underline\Delta_-)=1-e^{-1}\approx 0.632$, both implied by $c\in(0,1/2)$. Therefore, for any $c\in(0,1/2)$ and all $N\ge N_\star(q,c)$, the conclusions of Theorem~\ref{thm:separation} apply.

\section{Additional Bandit Results}
\label{app:bandit-extra}

\paragraph{Full $\eta$ sweep for the hidden-mode bandit.}
Table~\ref{tab:hidden-mode-eta-full} reports the complete $\eta$ sweep behind
Table~\ref{tab:bandit-main}. The main text focuses on $\IDSzero$ because
the structural separation theorem is stated for vanilla IDS and relies on the
boundary convention $g=0<\Delta\Rightarrow\Psi=+\infty$. The sweep shows why
this convention matters: when the residual mode probability is extremely small
($N=1000$), regularised $\IDSeta$ with $\eta>0$ can make the uninformative
default action appear cheap, while $\IDSzero$ still probes.

\begin{table}[h]
\centering\small
\caption{Full $\eta$ sweep on the hidden-mode bandit. Cumulative regret over
$T=500$ online steps, averaged over $10$ seeds. A value near $0.15$ means the
agent probed once and then acted optimally.}
\label{tab:hidden-mode-eta-full}
\begin{tabular}{rcccccccc}
\toprule
$N$ & $p_N$ & greedy & UCB & IDS$_{0}$ & IDS$_{0.01}$ & IDS$_{0.05}$ & IDS$_{0.1}$ & TS \\
\midrule
100  & 0.377  & 0.15  & 0.50  & 0.15 & 0.15 & 0.15 & 0.15  & 1.06 \\
200  & 0.268  & 0.15  & 0.15  & 0.15 & 0.15 & 0.15 & 0.15  & 1.28 \\
300  & 0.182  & 0.15  & 0.15  & 0.15 & 0.15 & 0.15 & 0.15  & 1.29 \\
1000 & 0.0066 & 3.31  & 3.31  & 0.15 & 3.31 & 3.31 & 3.31  & 1.79 \\
\bottomrule
\end{tabular}
\end{table}

\paragraph{Full $\eta$ sweep for the biased contextual bandit.}
Table~\ref{tab:linear-bandit-eta-full} reports the complete $\eta$ sweep behind
Table~\ref{tab:bandit-main}. The main text reports $\eta=0.5$ as a
representative regularised IDS setting because it is the strongest choice in
the most biased warm-start regimes ($N=20,50$). The sweep also shows that the
best regularisation level is regime-dependent: when the warm start is more
accurate ($N=100$), a smaller regulariser performs better, while large
regularisation over-emphasises immediate regret.

\begin{table}[h]
\centering\small
\caption{Full $\eta$ sweep for the biased linear contextual bandit. Final
cumulative regret over $T=200$, mean over $20$ seeds. Bold marks the best IDS
variant per row.}
\label{tab:linear-bandit-eta-full}
\setlength{\tabcolsep}{4pt}
\begin{tabular}{rcccccccc}
\toprule
$N$ & greedy & UCB & TS & IDS$_{0}$ & IDS$_{0.01}$ & IDS$_{0.05}$ & IDS$_{0.1}$ & IDS$_{0.5}$ \\
\midrule
20  & $42.85{\scriptstyle\pm 34.83}$ & $4.80{\scriptstyle\pm 1.37}$ & $10.64{\scriptstyle\pm 1.52}$ & $7.29{\scriptstyle\pm 1.79}$ & $6.17{\scriptstyle\pm 1.42}$ & $5.19{\scriptstyle\pm 1.23}$ & $4.69{\scriptstyle\pm 1.43}$ & $\mathbf{3.57{\scriptstyle\pm 1.67}}$ \\
50  & $18.90{\scriptstyle\pm 26.91}$ & $1.94{\scriptstyle\pm 1.02}$ & $4.73{\scriptstyle\pm 0.76}$  & $2.51{\scriptstyle\pm 0.65}$ & $2.16{\scriptstyle\pm 0.88}$ & $1.90{\scriptstyle\pm 0.91}$ & $2.18{\scriptstyle\pm 1.50}$ & $\mathbf{1.67{\scriptstyle\pm 1.89}}$ \\
100 & $0.88{\scriptstyle\pm 2.08}$   & $0.113{\scriptstyle\pm 0.33}$ & $0.445{\scriptstyle\pm 0.52}$ & $0.152{\scriptstyle\pm 0.27}$ & $\mathbf{0.126{\scriptstyle\pm 0.32}}$ & $0.194{\scriptstyle\pm 0.59}$ & $0.632{\scriptstyle\pm 1.94}$ & $0.641{\scriptstyle\pm 1.95}$ \\

\bottomrule
\end{tabular}
\end{table}

\section{Algorithm, Implementation, and Hyperparameters}
\label{app:implementation}

\begin{algorithm}[h]
\caption{Off-to-on IDS, deep-ensemble path (used for D4RL). The linear-Gaussian path is identical except that $\Delta$ and $g$ use the BLR closed forms.}
\label{alg:practical}

\textbf{Inputs.} Offline dataset $\Doff$; ensemble size $K$; noise variance $\sigma^{2}$; gain temperature $\alpha$; regulariser $\eta$; online steps $T$; candidates per step $M$; perturbation $\sigma_a$; action bounds $[-a_{\max},a_{\max}]$; clip $q_{\max}$; UTD ratio $U$; mix ratio $\rho$.

\smallskip
\textit{Stage 1 (offline pretraining).}
\begin{enumerate}[leftmargin=*,itemsep=2pt]
\item For $i=1,\dots,K'$, train $(\pi^{(i)}_\psi, Q^{(i)})$ on $\Doff$ with TD3+BC~\citep{fujimoto2021minimalist}, with each transition admitted to member $i$'s minibatch with probability $p_{\mathrm{boot}}=0.9$.
\item Repack the twin heads of each critic as $K = 2K'$ independent ensemble members $\{Q_k\}$. Choose $\pi_\psi := \pi^{(1)}_\psi$ as the anchor actor.
\item Calibrate $\sigma^2 \leftarrow \mathrm{Var}(r + \gamma\bar Q(s', \bar\pi(s')) - \bar Q(s,a))$ and $\alpha$ on an offline holdout, where $\bar Q$ is the ensemble mean and $\bar\pi$ is the target actor.
\end{enumerate}

\smallskip
\textit{Stage 2 (online IDS).} For $t=1,\dots,T$:
\begin{enumerate}[leftmargin=*,itemsep=2pt]
\item Observe $s_t$. Set anchor $\bar a \leftarrow \pi_\psi(s_t)$ and candidates $A_t = \{\bar a\} \cup \{\mathrm{clip}(\bar a + \sigma_a \xi_m, -a_{\max}, a_{\max})\}_{m=2}^{M}$. Draw a wider set $A_{V^\star}$ of size $M_{V^\star}$ from the same proposal.
\item For each $a\in A_t\cup A_{V^\star}$, query the ensemble at the inference path so that $\widehat Q_k(s_t,a) = \mathrm{clip}(Q_k(s_t,a),-q_{\max},q_{\max})$ (the clip is \emph{not} applied during training forward passes). Compute
\(\Delta(a) \leftarrow \tfrac{1}{K}\sum_k(\max_{a'\in A_{V^\star}}\widehat Q_k(s_t,a') - \widehat Q_k(s_t,a)),\;
g(a) \leftarrow \tfrac{1}{2}\log(1 + \alpha\mathrm{Var}_k[\widehat Q_k(s_t,a)]/\sigma^2).\)
\item Pick $a_t \leftarrow \argmin_{a\in A_t}\Delta(a)^2/(g(a)+\eta)$, add execution noise $\mathcal N(0,\sigma_{\mathrm{ex}}^2)$, step the environment.
\item Push $(s_t,a_t,r_t,s_{t+1},\mathrm{done}_t)$ to the online replay buffer.
\item Slow fine-tune (every step, $U$ gradient steps): sample a batch with mix ratio $\rho$ from the online replay and $1-\rho$ from $\Doff$; update each $Q_k$ with bootstrap masking against the TD target $r + \gamma\bar Q(s', \pi_\psi^{\mathrm{tgt}}(s'))$; update $\pi_\psi$ with $-\bar Q(s, \pi_\psi(s))/|\bar Q|_{\mathrm{mean}}$ (no BC term in the online phase); Polyak-update $\pi_\psi^{\mathrm{tgt}}$ with rate $\tau$.
\item Every $E$ env steps, evaluate $\pi_\psi^{\mathrm{tgt}}$ deterministically on $n_{\mathrm{eval}}$ rollouts and log the D4RL normalised score.
\end{enumerate}
\end{algorithm}

\paragraph{Path A details (linear contextual bandit).}
We use a fixed feature map $\phi(s,a)\in\R^d$ constructed as follows: state and action are concatenated with quadratic and interaction terms, projected by a random Gaussian matrix, passed through a $\tanh$ nonlinearity, and $\ell_2$-normalised. The resulting feature dimension is $d=128$.

Rewards follow a linear-Gaussian model
\[
r(s,a)=\phi(s,a)^\top w^\star+\varepsilon,\quad \varepsilon\sim\mathcal{N}(0,\sigma^2).
\]

The offline dataset is collected by a mismatched behaviour policy. We initialise a Bayesian linear regression posterior with prior precision $\lambda=1$ and noise variance $\sigma^2=1$, then update it using standard closed-form sufficient statistics:
\[
\Lambda = \lambda I + \sum \phi\phi^\top,\quad
\mu = \Lambda^{-1}\sum \phi r.
\]

During the online phase, the posterior is updated after each observation via rank-$1$ updates (Cholesky form). All methods share this posterior.

For IDS, we draw $S=64$ posterior samples to estimate $\Delta(a)$ by Monte Carlo. The information gain $g(a)$ is computed in closed form from the posterior covariance. The candidate set size is $M=64$.

\paragraph{Path B details (deep ensemble, D4RL).}
We use a TD3+BC-based ensemble of actor--critic models as a posterior surrogate. Offline training runs $K'=5$ independent actor--critic pairs with bootstrap masking probability $0.9$ for $5\times 10^5$ gradient steps. Each critic has twin Q-heads, which are repacked into a $K=10$ ensemble $\{Q_k\}$.

Online, we discard the BLR head and use the ensemble for action evaluation. After each environment step, both the actor and critics are updated using a replay buffer with a $50/50$ mix of offline and online data and an update-to-data ratio of $5$.

At each state, a candidate set $A_t$ (size $M=64$) is sampled around the actor output, and $V_k^\star(s)$ is computed using an enlarged set ($M_{V^\star}=256$). To stabilise disagreement estimates, we apply inference-time clipping
\[
\widehat Q_k(s,a)=\mathrm{clip}(Q_k(s,a),\pm q_{\max}),\quad q_{\max}=10^4.
\]

The IDS surrogate uses ensemble variance with scale parameters $(\sigma^2,\alpha)$ calibrated on a held-out subset of $\Doff$. Target networks use Polyak averaging with $\tau=0.005$.

\paragraph{Environment hyperparameters (D4RL).}
Three knobs are tuned per environment; all others are shared. The proposal sigma $\sigma_a$ controls the candidate spread, the IDS regulariser $\eta$ trades exploration against $\arg\min\Delta$, and the ensemble Q clip $q_{\max}$ caps OOD extrapolation. The choice $\eta=0.05$ was robust on every environment except \textsc{halfcheetah-medium-replay}, where multimodal data raises ensemble disagreement and the ratio $\Delta^2/(g+\eta)$ becomes unstable at small $\eta$; raising $\eta$ to $0.5$ recovered stable behaviour.

\begin{table}[h]
\centering\small
\caption{Per-environment hyperparameters for Path~B. All other settings are shared.}
\label{tab:perenv}
\begin{tabular}{lccc}
\toprule
Environment (\texttt{-v2}) & $\sigma_a$ & $\eta$ & $q_{\max}$ \\
\midrule
walker2d-medium / -replay     & 0.1 & 0.05 & $10^4$ \\
hopper-medium / -replay       & 0.1 & 0.05 & $10^4$ \\
halfcheetah-medium            & 0.1 & 0.05 & $10^3$ \\
halfcheetah-medium-replay     & 0.1 & 0.5  & $10^3$ \\
\bottomrule
\end{tabular}
\end{table}

Total env steps $T=10^{6}$; warmup $25{,}000$ steps; evaluation every $5{,}000$ env steps over $10$ deterministic rollouts using $\pi_\psi^{\mathrm{tgt}}$; action exploration noise $\mathcal{N}(0,0.1^2)$; slow fine-tune every step with $U=5$ gradient updates; learning rate $3\times 10^{-4}$ for both ensemble members and actor; online batch size $256$; offline-mix ratio $\rho=0.5$ on \textsc{halfcheetah} and $\rho=0.0$ on \textsc{walker2d}/\textsc{hopper}; online BC weight $0$; Polyak rate $\tau=0.005$; policy frequency $2$; gain temperature $\alpha$ for $g$ calibrated on a $5{,}000$-transition offline holdout against the empirical TD-residual variance; Q-clamp applied at inference only. The D4RL data is loaded through the original release~\citep{fu2020d4rl} at \texttt{-v2}. Each (env, seed) D4RL run uses one NVIDIA Ada6000 GPU; bandit experiments run on a single CPU under one hour per seed.

\end{document}